\documentclass[letterpaper]{article}
\usepackage{aaai}
\usepackage{times}
\usepackage{helvet}
\usepackage{courier}
\usepackage{url}

\usepackage{graphicx} 

\usepackage{algorithm}
\usepackage{algorithmic}

\usepackage{amsmath}
\usepackage{amsthm}
\usepackage{amsfonts}
\usepackage{MnSymbol} 

\newcommand{\prob}{\mathrm{Pr}}
\newcommand{\variance}{\mathrm{Var}}
\newcommand{\E}{\mathbb{E}}
\newcommand{\R}{\mathbb{R}}
\newcommand{\cvar}{\Phi_\alpha}
\newcommand{\var}{\nu_\alpha}

\newcommand{\cD}{\mathcal{D}}

\newcommand{\cX}{\mathcal{X}}
\newcommand{\cY}{\mathcal{Y}}
\newcommand{\cS}{\mathcal{S}}
\newcommand{\cA}{\mathcal{A}}
\newcommand{\cK}{\mathcal{K}}
\newcommand{\bx}{\mathbf{x}}
\newcommand{\bX}{\mathbf{X}}

\newcommand{\is}{\textrm{IS}}
\newcommand{\Dth}{\frac{\partial}{\partial \theta_j}}
\newcommand{\Dom}{\frac{\partial}{\partial \omega}}
\newcommand{\Dt}[1]{\frac{\partial {#1}}{\partial \theta_j}}
\newcommand{\T}{^{\top}}
\newcommand{\dcvar}{\Delta_{j;N}} 
\newcommand{\dcvarIS}{\Delta_{j;N}^{\is}} 
\newcommand{\evar}{\tilde{v}}  

\newcommand{\alggiven}{\textbf{Given:}}
\newcommand{\algreturn}{\textbf{Return:}}

\DeclareMathOperator*{\argmin}{arg\,min}

\newcommand{\citealt}[1]{\citeauthor{{#1}}, \citeyear{{#1}}}
\newcommand{\citet}[1]{\citeauthor{{#1}} (\citeyear{{#1}})}
\newcommand{\citep}[1]{\cite{{#1}}}

\newtheorem{theorem}{Theorem}
\newtheorem{assumption}{Assumption}
\newtheorem{proposition}[theorem]{Proposition}

\begin{document}

\title{Optimizing the CVaR via Sampling}

\author{
Aviv Tamar, Yonatan Glassner, \and Shie Mannor \\
Electrical Engineering Department\\
The Technion - Israel Institute of Technology\\
Haifa, Israel 32000\\
\{avivt, yglasner\}@tx.technion.ac.il, shie@ee.technion.ac.il \\
}

\maketitle

\begin{abstract}
\begin{quote}
Conditional Value at Risk (CVaR) is a prominent risk measure that is being used extensively in various domains. We develop a new formula for the gradient of the CVaR in the form of a conditional expectation. Based on this formula, we propose a novel sampling-based estimator for the gradient of the CVaR, in the spirit of the likelihood-ratio method. We analyze the bias of the estimator, and prove the convergence of a corresponding stochastic gradient descent algorithm to a local CVaR optimum. Our method allows to consider CVaR optimization in new domains. As an example, we consider a reinforcement learning application, and learn a risk-sensitive controller for the game of Tetris.
\end{quote}
\end{abstract}

\section{Introduction}\label{sec:intro}

Conditional Value at Risk (CVaR; \citealt{rockafellar2000optimization}) is an established risk measure that has found extensive use in finance among other fields.
For a random payoff $R$, whose distribution is parameterized by a controllable parameter $\theta$, the $\alpha$-CVaR is defined as the expected payoff over the $\alpha \%$ worst outcomes of $Z$:
\begin{equation*}
    \Phi(\theta) = \E^\theta \left[ \left. R \right| R \leq \nu_\alpha(\theta) \right],
\end{equation*}
where $\nu_\alpha(\theta)$ is the $\alpha$-quantile of $R$. CVaR optimization aims to find a parameter $\theta$ that maximizes $\Phi(\theta)$.

When the payoff is of the structure $R = f_\theta(X)$, where $f_\theta$ is a deterministic function, and $X$ is random but does not depend on $\theta$, CVaR optimization may be formulated as a stochastic program, and solved using various approaches \cite{rockafellar2000optimization,hong_simulating_2009,iyengar2013fast}. Such a payoff structure is appropriate for certain domains, such as portfolio optimization, in which the investment strategy generally does not affect the asset prices. However, in many important domains, for example queueing systems, resource allocation, and reinforcement learning, the tunable parameters also control the \emph{distribution} of the random outcomes. Since existing CVaR optimization methods are not suitable for such cases, and due to increased interest in risk-sensitive optimization recently in these domains \cite{tamar2012policy,prashanth2013actor}, there is a strong incentive to develop more general CVaR optimization algorithms.

In this work, we propose a CVaR optimization approach that is applicable when $\theta$ also controls the distribution of $X$. The basis of our approach is a new formula that we derive for the CVaR gradient $\frac{\partial \Phi(\theta)}{\partial \theta}$ in the form of a conditional expectation. Based on this formula, we propose a sampling-based estimator for the CVaR gradient, and use it to optimize the CVaR by stochastic gradient descent.

In addition, we analyze the bias of our estimator, and use the result to prove convergence of the stochastic gradient descent algorithm to a local CVaR optimum. Our method allows us to consider CVaR optimization in new domains. As an example, we consider a reinforcement learning application, and learn a risk-sensitive controller for the game of Tetris. To our knowledge, CVaR optimization for such a domain is beyond the reach of existing approaches. Considering Tetris also allows us to easily interpret our results, and show that we indeed learn sensible policies.

We remark that in certain domains, CVaR is often not maximized directly, but used as a constraint in an optimization problem of the form $\max_\theta \E^\theta [R] \text{ s.t. } \Phi(\theta) \geq b$. Extending our approach to such problems is straightforward, using standard penalty method techniques (see, e.g., \citealt{tamar2012policy}, and \citealt{prashanth2013actor}, for a such an approach with a variance-constrained objective), since the key component for these methods is the CVaR gradient estimator we provide here. Another appealing property of our estimator is that it naturally incorporates importance sampling, which is important when $\alpha$ is small, and the CVaR captures \emph{rare} events.

\vspace{-10pt}
\paragraph{Related Work}
Our approach is similar in spirit to the \emph{likelihood-ratio} method (LR; \citealt{glynn1990likelihood}), that estimates the gradient of the \emph{expected} payoff. The LR method has been successfully applied in diverse domains such as queueing systems, inventory management, and financial engineering \cite{Fu2006gradients}, and also in reinforcement learning (RL; \citealt{sutton_reinforcement_1998}), where it is commonly known as the \emph{policy gradient} method \cite{baxter2001infinite,peters_reinforcement_2008}. Our work extends the LR method to estimating the gradient of the CVaR of the payoff.

Closely related to our work are the studies of \citet{hong_simulating_2009} and \citet{scaillet_nonparametric_2004}, who proposed perturbation analysis style estimators for the gradient of the CVaR, for the setting mentioned above, in which $\theta$ {\em does not affect} the distribution of $X$. Indeed, their gradient formulae are different than ours, and do not apply in our setting.

LR gradient estimators for other risk measures have been proposed by \citet{borkar2001sensitivity} for exponential utility functions, and by \citet{tamar2012policy} for mean--variance. These measures, however, consider a very different notion of risk than the CVaR. For example, the mean--variance measure is known to underestimate the risk of rare, but catastrophic events \cite{agarwal2004risks}.

Risk-sensitive optimization in RL is receiving increased interest recently. A mean-variance criterion was considered by \citet{tamar2012policy} and \citet{prashanth2013actor}. \citet{morimura_nonparametric_2010} consider the expected return, with a CVaR based risk-sensitive policy for guiding the exploration while learning. Their method, however, does not scale to large problems. \citet{borkar2014risk} optimize a CVaR constrained objective using dynamic programming, by augmenting the state space with the accumulated reward. As such, that method is only suitable for a finite horizon and a small state-space, and \emph{does not scale-up} to problems such as the Tetris domain we consider. A function approximation extension of \cite{borkar2014risk} is mentioned, using a three time scales stochastic approximation algorithm. In that work, three different learning rates are decreased to 0, and convergence is determined by the slowest one, leading to an overall slow convergence. In contrast, our approach requires only a single learning rate.
Recently, \citet{prashanth2014cvar} used our gradient formula of Proposition \ref{prop:grad_general} (from a preliminary version of this paper) in a two time-scale stochastic approximation scheme to show convergence of CVaR optimization.
Besides providing the theoretical basis for that work, our current convergence result (Theorem \ref{thm:CVaRPG_convergence}) obviates the need for the extra time-scale, and results in a simpler and faster algorithm.

\section{A CVaR Gradient Formula}\label{sec:CVaR_Sensitivity}

In this section we present a new LR-style formula for the gradient of the CVaR. This gradient will be used in subsequent sections to optimize the CVaR with respect to some parametric family.
We start with a formal definition of the CVaR, and then present a CVaR gradient formula for 1-dimensional random variables. We then extend our result to the multi-dimensional case.

Let $Z$ denote a random variable with a cumulative distribution function (C.D.F.) $F_Z(z) = \prob(Z \leq z)$. For convenience, we assume that $Z$ is a continuous random variable, meaning that $F_Z(z)$ is everywhere continuous. We also assume that $Z$ is bounded.
Given a confidence level $\alpha\in (0,1)$, the $\alpha$-Value-at-Risk, (VaR; or $\alpha$-quantile) of $Z$ is denoted $\var(Z)$, and given by
\begin{equation}\label{eq:VaR_def}
\var (Z) = F_Z^{-1}(\alpha) \doteq \inf\left\{z:F_Z(z) \geq \alpha \right\}.
\end{equation}
The $\alpha$-Conditional-Value-at-Risk of $Z$ is denoted by $\cvar(Z)$ and defined as the expectation of the $\alpha$ fraction of the worst outcomes of $Z$
\begin{equation}\label{eq:CVaR_def}
\cvar (Z) = \E \left[ \left. Z \right| Z \leq \var (Z)\right].
\end{equation}

We next present a formula for the sensitivity of $\cvar(Z)$ to changes in $F_Z(z)$.
\subsection{CVaR Gradient of a 1-Dimensional Variable}\label{ssec:1D_formula}

Consider again a random variable $Z$, but now let its probability density function (P.D.F.) $f_Z(z;\theta)$ be parameterized by a vector $\theta \in \R^k$. We let $\var (Z;\theta)$ and $\cvar (Z;\theta)$ denote the VaR and CVaR of $Z$ as defined in Eq. \eqref{eq:VaR_def} and \eqref{eq:CVaR_def}, when the parameter is $\theta$, respectively.

We are interested in the sensitivity of the CVaR to the parameter vector, as expressed by the gradient $\Dth \cvar (Z;\theta)$. In all but the most simple cases, calculating the gradient analytically is intractable. Therefore, we derive a formula in which $\Dth \cvar (Z;\theta)$ is expressed as a conditional expectation, and use it to calculate the gradient by \emph{sampling}.
For technical convenience, we make the following assumption:
\begin{assumption}\label{assumption:bounded_cont_Z}
$Z$ is a continuous random variable, and bounded in $[-b,b]$ for all $\theta$.
\end{assumption}

We also make the following smoothness assumption on $\var (Z;\theta)$ and $\cvar (Z;\theta)$
\begin{assumption}\label{assumption:bounded_derivatives}
For all $\theta$ and $1 \!\leq \!j \!\leq \!k$, the gradients $\Dt{ \var (Z;\theta)}$ and $\Dt{ \cvar (Z;\theta)}$ exist and are bounded.
\end{assumption}
Note that since $Z$ is continuous, Assumption \ref{assumption:bounded_derivatives} is satisfied whenever $\Dth f_Z(z;\theta)$ is bounded. Relaxing Assumptions \ref{assumption:bounded_cont_Z} and \ref{assumption:bounded_derivatives} is possible, but involves technical details that would  complicate the presentation, and is left to future work.
The next assumption is standard in LR gradient estimates
\begin{assumption}\label{assumption:LR_assumption}
For all $\theta$, $z$, and $1 \leq j \leq k$, we have that $\Dt{ f_Z(z;\theta) }/ f_Z(z;\theta)$ exists and is bounded.
\end{assumption}

In the next proposition we present a LR-style sensitivity formula for $\cvar (Z;\theta)$, in which the gradient is expressed as a conditional expectation. In Section \ref{sec:CVaR_Algorithm} we shall use this formula to suggest a sampling algorithm for the gradient.
\begin{proposition}\label{prop:Grad}
Let Assumptions \ref{assumption:bounded_cont_Z}, \ref{assumption:bounded_derivatives}, and \ref{assumption:LR_assumption} hold. Then
\begin{equation*}
\Dt{ \cvar \!(\!Z;\!\theta)} \!=\!
\E^{\theta}\!\left[ \!\left.\Dt{ \log \! f_Z(Z;\theta)}(Z \!-\! \var (Z;\theta)\!)\right|\!Z \!\leq\! \var \!(Z;\!\theta)\! \right]\!\!.
\end{equation*}
\end{proposition}

\begin{proof}
Define the level-set $D_{\theta} \!=\! \left\{ \!z\in [-b,b] : z \leq \var\!(Z;\theta) \right\}.$
By definition, $D_\theta \equiv [-b, \var(Z;\theta) ]$, and
$
\int_{z\in D_{\theta}}f_Z\left(z;\theta\right)dz=\alpha.
$
Taking a derivative and using the Leibniz rule we obtain
\begin{equation}\label{eq:proof0}
\begin{split}
0 &= \Dth \int_{-b}^{\var(Z;\theta)}f_Z\left(z;\theta\right)dz \\
 &= \int_{-b}^{\var(Z;\theta)}\Dt{ f_Z\left(z;\theta\right)}dz
 +\Dt{ \var(Z;\theta)} f_Z\left(\var(Z;\theta);\theta\right).
 \end{split}\raisetag{3.4\baselineskip}
\end{equation}
By definition \eqref{eq:CVaR_def} we have
$
\cvar (Z;\theta)=\int_{z\in D_{\theta}}\frac{f_Z\left(z;\theta\right)z}{\alpha}dz
= \alpha^{-1}\int_{-b}^{\var(Z;\theta)}f_Z\left(z;\theta\right)z dz.
$
Now, taking a derivative and using the Leibniz rule we obtain
\begin{equation}\label{eq:proof1}
\begin{split}
\Dth \cvar (Z;\theta)
=& \alpha^{-1} \int_{-b}^{\var(Z;\theta)}\Dt{f_Z\left(z;\theta\right)}z dz \\
&\!+\! \alpha^{-1} \Dt{ \var(Z;\theta)} f_Z \! \left(\var(Z;\theta);\theta\right)\var(Z;\theta).
\end{split}\raisetag{3.4\baselineskip}
\end{equation}

Rearranging, and plugging \eqref{eq:proof0} in \eqref{eq:proof1} we obtain
$
\Dth \cvar (Z;\theta)=
\alpha^{-1} \int_{-b}^{\var(Z;\theta)}\Dt{ f_Z\left(z;\theta\right)}\left(z - \var(Z;\theta)\right)dz.
$
Finally, using the likelihood ratio trick -- multiplying and dividing
by $f_Z\left(z;\theta\right)$ inside the integral, which is justified due to Assumption \ref{assumption:LR_assumption}, we obtain the required expectation.
\end{proof}

Let us contrast the CVaR LR formula of Proposition \ref{prop:Grad} with the standard LR formula for the expectation \cite{glynn1990likelihood} $\Dth \E^{\theta}[Z] = \E^{\theta}\left[\Dt{ \log f_Z(Z;\theta)}(Z-b)\right]$, where the baseline $b$ could be any arbitrary constant. Note that in the CVaR case the baseline is \emph{specific}, and, as seen in the proof, accounts for the sensitivity of the level-set $D_{\theta}$.
Quite surprisingly, this specific baseline turns out to be exactly the VaR, $\var(Z;\theta)$, which, as we shall see later, also leads to an elegant sampling based estimator.

In a typical application, $Z$ would correspond to the performance of some system, such as the profit in portfolio optimization, or the total reward in RL. Note that in order to use Proposition \ref{prop:Grad} in a gradient estimation algorithm, one needs access to $\Dth \log f_Z(Z;\theta)$: the sensitivity of the performance distribution to the parameters. Typically, the system performance is a complicated function of a high-dimensional random variable. For example, in RL and queueing systems, the performance is a function of a trajectory from a stochastic dynamical system, and calculating its probability distribution is usually intractable. The sensitivity of the trajectory distribution to the parameters, however, is often easy to calculate, since the parameters typically control how the trajectory is generated. We shall now generalize Proposition \ref{prop:Grad} to such cases. The utility of this generalization is further exemplified in Section \ref{sec:RL}, for the RL domain.

\subsection{CVaR Gradient Formula -- General Case}\label{ssec:multi-D_formula}

Let $\bX=(X_1,X_2,\dots,X_n)$ denote an $n-$dimensional random variable with a finite support $[-b,b]^n$, and let $Y$ denote a discrete random variable taking values in some countable set $\cY$. Let $f_{Y}(y;\theta)$ denote the probability mass function of $Y$, and let $f_{\bX|Y}(\bx|y;\theta)$ denote the probability density function of $\bX$ given $Y$. Let the reward function $r$ be a bounded mapping from $[-b,b]^n \times \cY$ to $\R$, and consider the random variable $R \doteq r(\bX,Y)$. We are interested in a formula for $\Dth \cvar (R;\theta)$.

We make the following assumption, similar to Assumptions \ref{assumption:bounded_cont_Z}, \ref{assumption:bounded_derivatives}, and \ref{assumption:LR_assumption}.
\begin{assumption}\label{assumption:bounded_cont_rZ}
The reward $R$ is a continuous random variable for all $\theta$. Furthermore, for all $\theta$ and $1 \leq j \leq k$, the gradients $\Dth \var (R;\theta)$ and $\Dth \cvar (R;\theta)$ are well defined and bounded. In addition $\Dt{ \log f_{\bX|Y}\left(\bx|y;\theta\right)}$ and $\Dt{ \log f_{Y}\left(y;\theta\right)}$ exist and are bounded for all $\bx$, $y$, and $\theta$.
\end{assumption}

Define the level-set
$
\cD_{y;\theta} = \left\{ \bx \in [-b,b]^n : r(\bx,y)\leq \var(R;\theta) \right\}.
$
We require some smoothness of the function $r$, that is captured by the following assumption on $\cD_{y;\theta}$.
\begin{assumption}\label{assumption:D_finite_sum}
For all $y$ and $\theta$, the set $\cD_{\!y;\theta}$ may be written as a finite sum of $L_{y;\theta}$ disjoint, closed, and connected components $D_{\!y;\theta}^i$, each with positive measure:
$
\cD_{\!y;\theta} \!=\! \sum_{i=1}^{L_{y;\theta}} D_{\!y;\theta}^i.
$
\end{assumption}
Assumption \ref{assumption:D_finite_sum} may satisfied, for example, when $r(\bx,y)$ is Lipschitz in $\bx$ for all $y\in \cY$.
We now present a sensitivity formula for $\cvar (R;\theta)$.
\begin{proposition}\label{prop:grad_general}
Let Assumption \ref{assumption:bounded_cont_rZ} and \ref{assumption:D_finite_sum} hold. Then
\begin{equation*}
\begin{split}
&\Dth \cvar (R;\theta) = \E^{\theta}\left[\left(\frac{\partial {\log f_Y(Y;\theta)}}{\partial \theta_j}+ \right.\right.\\
&\left.\left.\left. \frac{\partial {\log f_{\bX|Y}(\bX|Y;\theta)}}{\partial \theta_j}\right)\left(R - \var (R;\theta)\right)\right| R  \leq  \var (R;\theta)\right].
\end{split}
\end{equation*}
\end{proposition}

The proof of Proposition \ref{prop:grad_general} is similar in spirit to the proof of Proposition \ref{prop:Grad}, but involves some additional difficulties of applying the Leibnitz rule in a multidimensional setting. It is given in \cite{tamar2014cvar}. We reiterate that relaxing Assumptions 4 and 5 is possible, but is technically involved, and left for future work. In the next section we show that the formula in Proposition \ref{prop:grad_general} leads to an effective algorithm for estimating $\Dth \cvar (R;\theta)$ by sampling.
\section{A CVaR Gradient Estimation Algorithm}\label{sec:CVaR_Algorithm}

The sensitivity formula in Proposition \ref{prop:grad_general} suggests a natural Monte--Carlo (MC) estimation algorithm. The method, which we label \texttt{GCVaR} (Gradient estimator for CVaR), is described as follows. Let $\bx_{1},y_1\dots,\bx_{N},y_N$ be $N$ samples drawn i.i.d. from $f_{\bX,Y}(\bx,y;\theta)$, the joint distribution of $\bX$ and $Y$.
We first estimate $\var(R;\theta)$ using the empirical $\alpha$-quantile\footnote{Algorithmically, this is equivalent to first sorting the $r(\bx_i,y_i)$'s in ascending order, and then selecting $\evar$ as the $\lceil \alpha N \rceil$ term in the sorted list.} $\evar$
\begin{equation}\label{eq:empirical_VaR}
\evar = \inf_z \hat{F}(z) \geq \alpha,
\end{equation}
where $\hat{F}(z)$ is the empirical C.D.F. of $R$:
$
\hat{F}(z) \doteq \frac{1}{N} \sum_{i=1}^{N} \mathbf{1}_{r(\bx_{i},y_i)\leq z}.
$
The MC estimate of the gradient $\dcvar \approx \Dth\cvar (R;\theta)$ is given by
\begin{equation}\label{eq:simple_alg}
\begin{split}
\dcvar =& \frac{1}{\alpha N}\!\sum_{i=1}^{N} \!\left(\Dt{ \log f_{Y}\left(y_i;\theta\right)} \!+\! \Dt{ \log f_{\bX|Y}\left(\bx_{i}|y_i;\theta\right)}\right) \!\!\times\\
& \times\! \left( r(\bx_{i},y_i)-\evar \right) \mathbf{1}_{r(\bx_{i},y_i)\leq \evar}.
\end{split}\raisetag{0.8\baselineskip}
\end{equation}

\begin{algorithm} \label{alg:GCVaR}
\caption{\texttt{GCVaR}}
1: \alggiven
\begin{itemize}
\item CVaR level $\alpha$
\item A reward function $r(\bx,y):\R^n\times \cY \to \R$
\item Derivatives $\Dt{}$ of the probability mass function $f_Y\left(y;\theta\right)$ and probability density function $f_{\bX|Y}\left(\bx|y;\theta\right)$
\item An i.i.d. sequence $\bx_{1},y_1,\dots,\bx_{N},y_N \sim f_{\bX,Y}(\bx,y;\theta)$.
\end{itemize}
2: Set $r_1^s,\dots,r_N^s = \textrm{Sort}\left(r(\bx_{1},y_1),\dots,r(\bx_{N},y_N)\right)$

3: Set $\evar = r_{\lceil \alpha N \rceil}^s$

4: For $j = 1,\dots,k$ do
\begin{equation*}
\begin{split}
\dcvar =& \frac{1}{\alpha N}\! \sum_{i=1}^{N} \! \left(\Dt{ \log f_{Y}\left(y_i;\theta\right)} \!+\! \Dt{ \log f_{\bX|Y}\left(\bx_{i}|y_i;\theta\right)}\right)\!\! \times\\
& \times \! \left( r(\bx_{i},y_i)-\evar \right) \mathbf{1}_{r(\bx_{i},y_i)\leq \evar}
\end{split}
\end{equation*}
5: \algreturn $\Delta_{1;N},\dots,\Delta_{k;N}$
\end{algorithm}

It is known that the empirical $\alpha$-quantile is a biased estimator of $\var(R;\theta)$. Therefore, $\dcvar$ is also a biased estimator of $\Dth\cvar (R;\theta)$. In the following we analyze and bound this bias.
We first show that $\dcvar$ is a consistent estimator. The proof is similar to the proof of Theorem 4.1 in \cite{hong_simulating_2009}, and given in the supplementary material.

\begin{theorem}\label{thm:consistent}
Let Assumption \ref{assumption:bounded_cont_rZ} and \ref{assumption:D_finite_sum} hold. Then $\dcvar \to \Dth\cvar (R;\theta)$ w.p. 1 as $N \to \infty$.
\end{theorem}

With an additional smoothness assumption we can explicitly bound the bias.
Let $f_R(\cdot;\theta)$ denote the P.D.F. of $R$, and define the function
$
g(\beta;\theta) \doteq \E^\theta \!\!\left[ \!\left. \left(\frac{\partial {\log f_Y(Y;\theta)}}{\partial \theta_j}\!+\!\frac{\partial {\log f_{\bX|Y}(\bX|Y;\theta)}}{\partial \theta_j}\right)\!\!(R \!-\! \var (R;\theta)\!) \right| \! R=\beta\right]\!\!.
$
\begin{assumption}\label{assumption:nice_f_g}
For all $\theta$, $f_R(\cdot;\theta)$ and $g(\cdot;\theta)$ are continuous at $\var (R;\theta)$, and $f_R(\var (R;\theta);\theta) > 0$.
\end{assumption}

Assumption \ref{assumption:nice_f_g} is similar to Assumption 4 of \cite{hong_simulating_2009}, and may be satisfied, for example, when $\Dt{ \log f_{\bX|Y}\left(\bx|y;\theta\right)}$ is continuous and $r(\bx,y)$ is Lipschitz in $\bx$.
The next theorem shows that the bias is $\mathcal{O}(N^{-1/2})$. The proof, given in the supplementary material, is based on separating the bias to a term that is bounded using a result of \citet{hong_simulating_2009}, and an additional term that we bound using well-known results for the bias of empirical quantiles.
\begin{theorem}\label{thm:bias_bound}
Let Assumptions \ref{assumption:bounded_cont_rZ}, \ref{assumption:D_finite_sum}, and \ref{assumption:nice_f_g} hold. Then $\E \left[ \dcvar \right] - \Dth\cvar (R;\theta)$ is $O(N^{-1/2})$.
\end{theorem}

At this point, let us again contrast \texttt{GCVaR} with the standard LR method. One may naively presume that applying a standard LR gradient estimator to the $\alpha \%$ worst samples would work as a CVaR gradient estimator. This corresponds to applying the \texttt{GCVaR} algorithm without subtracting the $\evar$ baseline from the reward in \eqref{eq:simple_alg}. Theorems \ref{thm:consistent} and \ref{thm:bias_bound} show that such an estimator \emph{would not be consistent}. In fact, in the supplementary material we give an example where the gradient error of such an approach may be arbitrarily large.

In the sequel, we use \texttt{GCVaR} as part of a stochastic gradient descent algorithm
for CVaR optimization. An asymptotically decreasing gradient bias, as may be established from Theorem \ref{thm:consistent}, is necessary to guarantee convergence of such a procedure. Furthermore, the bound of Theorem \ref{thm:bias_bound} will allow us to quantify how many samples are needed at each iteration for such convergence to hold.

\subsection*{Variance Reduction by Importance Sampling}\label{ssec:IS}
For very low quantiles, i.e., $\alpha$ close to $0$, the \texttt{GCVaR} estimator would suffer from a high variance, since the averaging is effectively only over $\alpha N$ samples. This is a well-known issue in sampling based approaches to VaR and CVaR estimation, and is often mitigated using variance reduction techniques such as
Importance Sampling (IS; \citealt{rubinstein2011simulation}; \citealt{bardou2009computing}). In IS, the variance of a MC estimator is reduced by using samples from a \emph{different} sampling distribution, and suitably modifying the estimator to keep it unbiased. It is straightforward to incorporate IS into LR gradient estimators in general, and to our \texttt{GCVaR} estimator in particular. Due to space constraints, and since this is fairly standard textbook material (e.g., \citealt{rubinstein2011simulation}), we provide the full technical details in the supplementary material. In our empirical results we show that using IS indeed leads to significantly better performance.

\section{CVaR Optimization}\label{sec:CVaR_OPT}
In this section, we consider the setting of Section \ref{ssec:multi-D_formula}, and aim to solve the CVaR optimization problem:
\begin{equation}\label{eq:CVaR_optimization}
\max_{\theta \in \R^{k}} \cvar (R;\theta).
\end{equation}
For this goal we propose \texttt{CVaRSGD}: a stochastic gradient descent algorithm, based on the \texttt{GCVaR} gradient estimator. We now describe the \texttt{CVaRSGD} algorithm in detail, and show that it converges to a local optimum of \eqref{eq:CVaR_optimization}.

In \texttt{CVaRSGD}, we start with an arbitrary initial parameter $\theta^0\in \R^{k}$. The algorithm proceeds iteratively as follows. At each iteration $i$ of the algorithm, we first sample $n_i$ i.i.d. realizations $x_1,y_1,\dots,x_{n_i},y_{n_i}$ of the random variables $\bX$ and $Y$, from the distribution $f_{\bX,Y}(\bx,y;\theta^i)$. We then apply the \texttt{GCVaR} algorithm to obtain an estimate $\Delta_{j;n_i}$ of $\Dth \cvar (R;\theta^i)$, using the samples $x_1,y_1,\dots,x_{n_i},y_{n_i}$. Finally, we update the parameter according to
\begin{equation}\label{eq:theta_update_SGD}
\theta_{j}^{i+1} = \Gamma \left( \theta_{j}^{i} + \epsilon_i \Delta_{j;n_i} \right),
\end{equation}
where $\epsilon_i$ is a positive step size, and $\Gamma: \R^k \to \R^k$ is a projection to some compact set $\Theta$ with a smooth boundary.
The purpose of the projection is to facilitate convergence of the algorithm, by guaranteeing that the iterates remain bounded (this is a common stochastic approximation technique; \citealt{kushner2003stochastic}). In practice, if $\Theta$ is chosen large enough so that it contains the local optima of $\cvar (R;\theta)$, the projection would rarely occur, and would have a negligible effect on the algorithm.
Let $\hat{\Gamma}_{\theta}(\nu)\doteq \lim_{\delta\to 0}\frac{\Gamma(\theta + \delta \nu) - \theta}{\delta}$ denote an operator that, given a direction of change $\nu$ to the parameter $\theta$, returns a modified direction that keeps $\theta$ within $\Theta$. Consider the following ordinary differential equation:
\begin{equation}\label{eq:ODE_SGD}
    \dot{\theta} = \hat{\Gamma}_{\theta}\left( \nabla \cvar (R;\theta) \right), \quad \theta(0) \in \Theta.
\end{equation}
Let $\cK$ denote the set of all asymptotically stable equilibria of \eqref{eq:ODE_SGD}. The next theorem shows that under suitable technical conditions, the \texttt{CVaRSGD} algorithm converges to $\cK$ almost surely. The theorem is a direct application of Theorem 5.2.1 of \citet{kushner2003stochastic}, and given here without proof.

\begin{theorem}\label{thm:CVaRPG_convergence}
Consider the \texttt{CVaRSGD} algorithm \eqref{eq:theta_update_SGD}. Let Assumptions \ref{assumption:bounded_cont_rZ}, \ref{assumption:D_finite_sum}, and \ref{assumption:nice_f_g} hold, and assume that $\cvar (R;\theta)$ is continuously differentiable in $\theta$. Also, assume that $\sum_{i=1}^\infty \epsilon_i = \infty$, $\sum_{i=1}^\infty \epsilon_i^2 < \infty$, and that $\sum_{i=1}^\infty \epsilon_i \left| \E \left[ \Delta_{j;n_i} \right] - \Dth\cvar (R;\theta^i)\right|< \infty$ w.p. 1 for all $j$. Then $\theta^i \to \cK$ almost surely.
\end{theorem}

Note that from the discussion in Section \ref{sec:CVaR_Algorithm}, the requirement $\sum_{i=1}^\infty \epsilon_i \left| \E \left[ \Delta_{j;n_i} \right] - \Dth\cvar (R;\theta^i)\right|< \infty$ implies that we must have $\lim_{i \to \infty} n_i = \infty$. However, the rate of $n_i$ could be very slow, for example, using the bound of Theorem \ref{thm:bias_bound} the requirement may be satisfied by choosing $\epsilon_i = 1/i$ and $n_i = (\log i)^4$.


\section{Application to Reinforcement Learning}\label{sec:RL}

In this section we show that the \texttt{CVaRSGD} algorithm may be used in an RL policy-gradient type scheme, for optimizing performance criteria that involve the CVaR of the total return. We first describe some preliminaries and our RL setting, and then describe our algorithm.

We consider an episodic\footnote{Also known as a stochastic shortest path \cite{Ber2012DynamicProgramming}.} Markov Decision Problem (MDP) in discrete time with a finite state space $\cS$ and a finite action space $\cA$.
At time $t \in \{0,1,2,\dots\}$ the state is $s_t$, and an action $a_t$ is chosen according to a parameterized policy $\pi_\theta$, which assigns a distribution over actions $f_{a|h}(a|h;\theta)$ according to the observed history of states $h_t = s_0,\dots,s_t$. Then, an immediate random reward $\rho_t \sim f_{\rho|s,a} (\rho|s,a)$ is received,
and the state transitions to $s_{t+1}$ according to the MDP transition probability $f_{s'|s,a}(s'|s,a)$. We denote by $\zeta_0$ the initial state distribution and by $s^*$ a terminal state, and we assume that for all $\theta$, $s^*$ is reached w.p. 1.

For some policy $\pi_\theta$, let $s_0,a_0,\rho_0,s_1,a_1,\rho_1,\dots,s_\tau$ denote a state-action-reward trajectory from the MDP under that policy, that terminates at time $\tau$, i.e., $s_\tau = s^*$. The trajectory is a random variable, and we decompose\footnote{This decomposition is not restrictive, and used only to illustrate the definitions of Section \ref{sec:CVaR_Sensitivity}. One may alternatively consider a continuous state space, or discrete rewards, so long as Assumptions \ref{assumption:bounded_cont_rZ}, \ref{assumption:D_finite_sum}, and \ref{assumption:nice_f_g} hold.} it into a discrete part $Y \doteq s_0,a_0,s_1,a_1,\dots,s^*$ and a continuous part $X \doteq \rho_0,\rho_1,\dots,\rho_{\tau-1}$. Our quantity of interest is the total reward along the trajectory $R \doteq \sum_{t=0}^{\tau} \rho_t$.
In standard RL, the objective is to find the parameter $\theta$ that maximizes the expected return $V(\theta) = \E^\theta \left[ R \right]$. Policy gradient methods \cite{baxter2001infinite,MarTsi98,peters_reinforcement_2008} use simulation to estimate $\partial V(\theta) / \partial \theta_j$, and then perform stochastic gradient ascent on the parameters $\theta$. In this work we are risk-sensitive, and our goal is to \emph{maximize the CVaR of the total return}
$
J(\theta) \doteq \cvar (R;\theta).
$
In the spirit of policy gradient methods, we estimate $\partial J(\theta) / \partial \theta_j$ from simulation, using \texttt{GCVaR}, and optimize $\theta$ using \texttt{CVaRSGD}. We now detail our approach.

%
%

First, it is well known \cite{MarTsi98} that by the Markov property of the state transitions:
\begin{equation}\label{eq:MDP_dlogf}
\partial \log f_{Y}\left(Y;\theta\right) / \partial \theta = \sum_{t=0}^{\tau-1} \partial \log f_{a|h}(a_t|h_t;\theta)/ \partial \theta.
\end{equation}
Also, note that in our formulation we have
\begin{equation}\label{eq:MDP_dlogf2}
    \partial \log f_{\bX|Y}\left(x_{i}|y_i;\theta\right) / \partial \theta=0,
\end{equation}
since the reward does not depend on $\theta$ directly.

To apply \texttt{CVaRSGD} in the RL setting, at each iteration $i$ of the algorithm we simulate $n_i$ trajectories $x_1,y_1,\dots,x_{n_i},y_{n_i}$ of the MDP using policy $\pi_{\theta^i}$ (each $x_k$ and $y_k$ here together correspond to a single trajectory, as realizations of the random variables $X$ and $Y$ defined above). We then apply the \texttt{GCVaR} algorithm to obtain an estimate $\Delta_{j;n_i}$ of $\partial J(\theta) / \partial \theta_j$, using the simulated trajectories $x_1,y_1,\dots,x_{n_i},y_{n_i}$, Eq. \eqref{eq:MDP_dlogf}, and Eq. \eqref{eq:MDP_dlogf2}.
Finally, we update the policy parameter according to Eq. \eqref{eq:theta_update_SGD}.
Note that due to Eq. \eqref{eq:MDP_dlogf}, the transition probabilities of the MDP, which are generally not known to the decision maker, are not required for estimating the gradient using \texttt{GCVaR}. Only policy-dependent terms are required.

We should remark that for the standard RL criterion $V(\theta)$, a Markov policy that depends only on the current state suffices to achieve optimality \cite{Ber2012DynamicProgramming}. For the CVaR criterion this is not necessarily the case. \citet{bauerle2011markov} show that under certain conditions, an augmentation of the current state with a function of the accumulated reward suffices for optimality. In our simulations, we used a Markov policy, and still obtained useful and sensible results.

Assumptions \ref{assumption:bounded_cont_rZ}, \ref{assumption:D_finite_sum}, and \ref{assumption:nice_f_g}, that are required for convergence of the algorithm, are reasonable for the RL setting, and may be satisfied, for example, when $f_{\rho|s,a} (\rho|s,a)$ is smooth, and $\partial\log f_{a|h}(a|h;\theta) / \partial \theta_j$ is well defined and bounded. This last condition is standard in policy gradient literature, and a popular policy representation that satisfies it is softmax action selection \cite{sutton_policy_2000,MarTsi98}, given by
$
f_{a|h}(a|h;\theta) = \frac{\exp ( \phi(h,a)^{\T} \theta)}{\sum_{a'} \exp(\phi(h,a')^{\T} \theta)},
$
where $\phi(h,a)\in \R^k$ are a set of $k$ features that depend on the history and action.

In some RL domains, the reward takes only discrete values. While this case is not specifically covered by the theory in this paper, one may add an arbitrarily small smooth noise to the total reward for our results to hold. Since such a modification has negligible impact on performance, this issue is of little importance in practice. In our experiments the reward was discrete, and we did not observe any problem.

\begin{figure*}[ht]
\centering
\includegraphics[width=\textwidth]{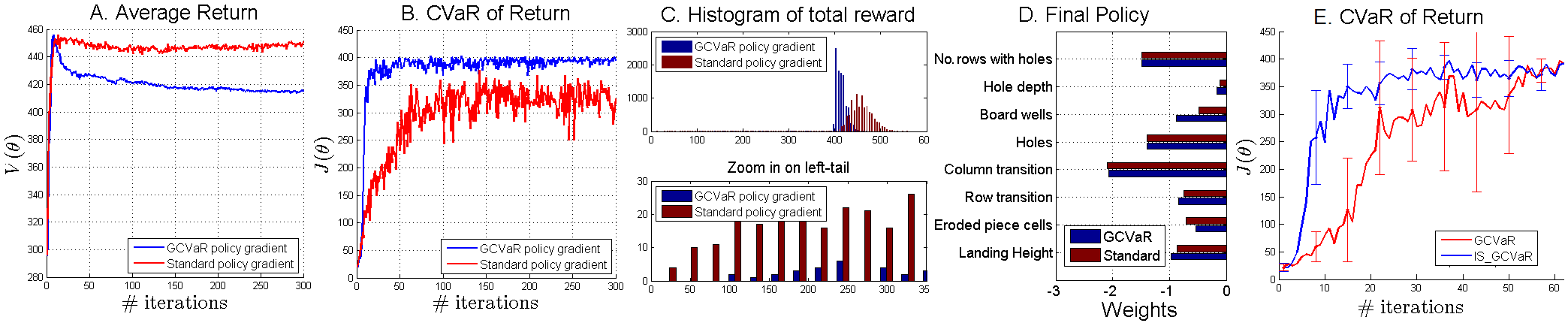}
\vskip -0.1in
\caption{\textbf{\texttt{GCVaR} vs. policy gradient.} (A,B) Average return (A) and CVaR ($\alpha=0.05$) of the return (B) for \texttt{CVaRSGD} and standard policy-gradient vs. iteration. (C) Histogram (counts from 10,000 independent runs) of the total return of the final policies. The lower plot is a zoom-in on the left-tail, and clearly shows the risk-averse behavior of the \texttt{CVaRSGD} policy. (D) Final policy parameters. Note the difference in the Board Well feature, which encourages risk taking. (E) CVaR ($\alpha=0.01$) of the return for \texttt{CVaRSGD} vs. iteration, with and without importance sampling.}
\label{fig1}
\vskip -0.1in
\end{figure*}

\subsection{Experimental Results}\label{sec:experiments}
We examine Tetris as a test case for our algorithms. Tetris is a popular RL benchmark that has been studied extensively. The main challenge in Tetris is its large state space, which necessitates some form of approximation in the solution technique. Many approaches to learning controllers for Tetris are described in the literature, among them are approximate value iteration \cite{tsitsiklis1996feature}, policy gradients \cite{kakade2001natural,furmston2012unifying}, and modified policy iteration \cite{gabillon2013approximate}.
The standard performance measure in Tetris is the expected number of cleared lines in the game. Here, we are interested in a risk-averse performance measure, captured by the CVaR of the total game score. Our goal in this section is to compare the performance of a policy optimized for the CVaR criterion versus a policy obtained using the standard policy gradient method. As we will show, optimizing the CVaR indeed produces a different policy, characterized by a risk-averse behavior. We note that at present, the best results in the literature (for the standard performance measure) were obtained using a modified policy iteration approach \cite{gabillon2013approximate}, and not using policy gradients. We emphasize that our goal here is not to compete with those results, but rather to illustrate the application of \texttt{CVaRSGD}. We do point out, however, that whether the approach of \citet{gabillon2013approximate} could be extended to handle a CVaR objective is currently not known.


We used the regular  $10\times 20$ Tetris board with the 7 standard shapes (a.k.a. \textit{tetrominos}). In order to induce risk-sensitive behavior, we modified the reward function of the game as follows. The score for clearing 1,2,3 and 4 lines is 1,4,8 and 16 respectively. In addition, we limited the maximum number of steps in the game to 1000. These modifications strengthened the difference between the risk-sensitive and nominal policies, as they induce a tradeoff between clearing many 'single' lines with a low profit, or waiting for the more profitable, but less frequent, 'batches'.


We used the softmax policy, with the feature set of \citet{thiery2009improvements}.
Starting from a fixed policy parameter $\theta_0$, which was obtained by running several iterations of standard policy gradient (giving both methods a 'warm start'), we ran both \texttt{CVaRSGD} and standard policy gradient\footnote{Standard policy gradient is similar to \texttt{CVaRSGD} when $\alpha = 1$. However, it is common to subtract a baseline from the reward in order to reduce the variance of the gradient estimate. In our experiments, we used the average return $<r>$ as a baseline, and our gradient estimate was $\frac{1}{N}\sum_{i=1}^{N} \Dt{ \log f_Y\left(y_{i};\theta\right)}(r(x_i,y_{i})-<r>)$.} for enough iterations such that both algorithms (approximately) converged. We set $\alpha = 0.05$ and $N = 1000$.

In Fig.~\ref{fig1}A and Fig.~\ref{fig1}B we present the average return $V(\theta)$ and CVaR of the return $J(\theta)$ for the policies of both algorithms at each iteration (evaluated by MC on independent trajectories). Observe that for \texttt{CVaRSGD}, the average return has been compromised for a higher CVaR value.

This compromise is further explained in Fig.~\ref{fig1}C, where we display the reward distribution of the final policies.
It may be observed that the left-tail distribution of the CVaR policy is significantly lower than the standard policy. For the risk-sensitive decision maker, such results are very important, especially if the left-tail contains catastrophic outcomes, as is common in many real-world domains, such as finance.
To better understand the differences between the policies, we compare the final policy parameters $\theta$ in Fig.~\ref{fig1}D. The most significant difference is in the parameter that corresponds to the Board Well feature. A \textit{well} is a succession of unoccupied cells in a column, such that their left and right cells are both occupied. The controller trained by \texttt{CVaRSGD} has a smaller negative weight for this feature, compared to the standard controller, indicating that actions which create deep-wells are repressed. Such wells may lead to a high reward when they get filled, but are risky as they heighten the board.

To demonstrate the importance of IS in optimizing the CVaR when $\alpha$ is small, we chose $\alpha=0.01$, and $N = 200$, and compared \texttt{CVaRSGD} against its IS version, \texttt{IS\_CVaRSGD}, described in the supplementary material.
As Fig.~\ref{fig1}E shows, \texttt{IS\_GCVaRSGD} converged significantly faster, improving the convergence rate by more than a factor of 2. The full details are provided in the supplementary material.

\section{Conclusion and Future Work}
We presented a novel LR-style formula for the gradient of the CVaR performance criterion. Based on this formula, we proposed a sampling-based gradient estimator, and a stochastic gradient descent procedure for CVaR optimization that is guaranteed to converge to a local optimum. To our knowledge, this is the first extension of the LR method to the CVaR performance criterion, and our results extend CVaR optimization to new domains.

We evaluated our approach empirically in an RL domain: learning a risk-sensitive policy for Tetris. To our knowledge, such a domain is beyond the reach of existing CVaR optimization approaches. Moreover, our empirical results show that optimizing the CVaR indeed results in useful risk-sensitive policies, and motivates the use of simulation-based optimization for risk-sensitive decision making.




\section*{Acknowledgments}
The authors thank Odalric-Ambrym Maillard for many helpful discussions. The research leading to these results has received funding from the European Research Council under the European Union's Seventh Framework Program (FP/2007-2013) / ERC Grant Agreement n. 306638.

\newpage
\small
\bibliographystyle{aaai}
\bibliography{my_library}

\newpage
\normalsize
\onecolumn
\appendix
\section{Proof of Proposition \ref{prop:grad_general}}\label{supp:grad_proof}
\begin{proof}
The main difficulty in extending the proof of Proposition \ref{prop:Grad} to this case is in applying the Leibnitz rule in a multi-dimensional case. Such an extension is given by \citep{flanders1973differentiation}, which we now state.

We are given an $n-$dimensional $\theta-$dependent chain (field of integration) $D_\theta$ in $\R^n$. We also have an exterior differential $n-$form whose coefficients are $\theta$-dependent:
\begin{equation*}
\omega = f(\bx,\theta)dx_1 \wedge \cdots \wedge dx_n
\end{equation*}

The general Leibnitz rule\footnote{The formula in \citep{flanders1973differentiation} is for a more general case where $D_\theta$ is not necessarily $n-$dimensional. That formula includes an additional term $\int_{D_\theta} \mathbf{v} \righthalfcup d_{\bx}\omega$, where $d_{\bx}$ is the exterior derivative, which cancels in our case.} is given by
\begin{equation}\label{eq:general_Leibnitz}
\frac{\partial}{\partial\theta} \int_{D_\theta} \omega = \int_{\partial D_\theta} \mathbf{v} \righthalfcup \omega + \int_{D_\theta} \frac{\partial \omega}{\partial \theta}
\end{equation}
where $\mathbf{v}$ denotes the vector field of velocities $\frac{\partial }{\partial \theta} \bx$ of $D_\theta$, and $\mathbf{v} \righthalfcup \omega$ denotes the interior product between $\mathbf{v}$ and $\omega$ (see \citep{flanders1973differentiation} for more details).

We now write the CVaR explicitly as
\begin{equation*}
\begin{split}
\cvar (R;\theta) &= \frac{1}{\alpha} \sum_{y\in \cY}f_Y(y;\theta)\int_{\bx \in \cD_{y;\theta}} f_{\bX|Y}(\bx|y;\theta) r(\bx,y) d\bx \\
 &= \frac{1}{\alpha} \sum_{y\in \cY}f_Y(y;\theta)\sum_{i=1}^{L_{y;\theta}}\int_{\bx \in D_{y;\theta}^i} f_{\bX|Y}(\bx|y;\theta) r(\bx,y) d\bx,
\end{split}
\end{equation*}
therefore
\begin{equation}\label{eq:proof_cvar_derivtive}
\begin{split}
\Dth \cvar (R;\theta) &= \frac{1}{\alpha} \sum_{y\in \cY}\Dt{f_Y(y;\theta)} \sum_{i=1}^{L_{y;\theta}}\int_{\bx \in D_{y;\theta}^i} f_{\bX|Y}(\bx|y;\theta) r(\bx,y) d\bx \\
&+ \frac{1}{\alpha} \sum_{y\in \cY}f_Y(y;\theta) \sum_{i=1}^{L_{y;\theta}} \Dth \int_{\bx \in D_{y;\theta}^i} f_{\bX|Y}(\bx|y;\theta) r(\bx,y) d\bx
\end{split}
\end{equation}
We now treat each $D_{y;\theta}^i$ in the last sum separately. Let $\cX$ denote the set $[-b,b]^n$ over which $\bX$ is defined. Obviously, $D_{y;\theta}^i \subset \cX$.

We now make an important observation. By definition of the level-set $D_{y;\theta}^i$, and since it is closed by Assumption \ref{assumption:D_finite_sum}, for every $\bx \in \partial D_{y;\theta}^i$ we have that either
\begin{equation}\label{eq:proof_def_a}
\textrm{(a) } r(\bx,y) = \var (R;\theta),
\end{equation}
or
\begin{equation}\label{eq:proof_def_b}
\textrm{(b) } \bx \in \partial \cX, \textrm{ and } r(\bx,y) < \var (R;\theta).
\end{equation}
We write $\partial D_{y;\theta}^i = \partial D_{y;\theta}^{i,a} + \partial D_{y;\theta}^{i,b}$ where the two last terms correspond to the two possibilities in \eqref{eq:proof_def_a} and \eqref{eq:proof_def_b}.

We now claim that for the boundary term $\partial D_{y;\theta}^{i,b}$, we have
\begin{equation}\label{eq:zero_velocity}
\int_{\partial D_{y;\theta}^{i,b}} \mathbf{v} \righthalfcup \omega = 0.
\end{equation}
To see this, first note that by definition of $\cX$, the boundary $\partial \cX$ is smooth and has a unique normal vector at each point, except for a set of measure zero (the corners of $\cX$). Let $\partial \tilde{D}_{y;\theta}^{i,b}$ denote the set of all points in $\partial D_{y;\theta}^{i,b}$ for which a unique normal vector exists. For each $\bx \in \partial \tilde{D}_{y;\theta}^{i,b}$ we let $\mathbf{v}_\perp$ and $\mathbf{v}_\parallel$ denote the normal and tangent (with respect to $\partial \cX$) elements of the velocity $\frac{\partial }{\partial \theta} \bx$ at $\bx$, respectively. Thus,
\begin{equation*}
\mathbf{v} = \mathbf{v}_\perp + \mathbf{v}_\parallel.
\end{equation*}
For some $\epsilon > 0$ let $d_\epsilon$ denote the set $\left\{ \bx \in \partial D_{y;\theta}^{i,b} : r(\bx,y) < \var (R;\theta) - \epsilon \right\}$. From Assumption \ref{assumption:bounded_cont_rZ} we have that $\Dth \var (R;\theta)$ is bounded, therefore there exists $\delta(\epsilon) > 0$ such that for all $\theta'$ that satisfy $\| \theta - \theta' \| < \delta(\epsilon)$ we have $\left|\var (R;\theta') - \var (R;\theta)\right| < \epsilon$, and therefore $d_\epsilon \in \partial D_{y;\theta'}^{i,b}$. Since this holds for every $\epsilon > 0$, we conclude that a small change in $\theta$ does not change $\partial D_{y;\theta}^{i,b}$, and therefore we have
\begin{equation*}
\mathbf{v}_\perp = 0, \quad \forall \bx \in \partial \tilde{D}_{y;\theta}^{i,b}.
\end{equation*}
Furthermore, by definition of the interior product we have
\begin{equation*}
\mathbf{v}_\parallel \righthalfcup \omega = 0.
\end{equation*}
Therefore we have
\begin{equation*}
\int_{\partial D_{y;\theta}^{i,b}} \mathbf{v} \righthalfcup \omega = \int_{\partial \tilde{D}_{y;\theta}^{i,b}} \mathbf{v} \righthalfcup \omega = \int_{\partial \tilde{D}_{y;\theta}^{i,b}} \mathbf{v}_\parallel \righthalfcup \omega = 0,
\end{equation*}
and the claim follows.

Now, let $\omega_y = f_{\bX|Y}(\bx|y;\theta) r(\bx,y) dx_1 \wedge \cdots \wedge dx_n$. Using \eqref{eq:general_Leibnitz}, we have
\begin{equation}\label{eq:proof_omega_leibnitz}
\begin{split}
\Dth \int_{\bx \in D_{y;\theta}^i} \omega_y &= \int_{\partial D_{y;\theta}^i} \mathbf{v} \righthalfcup \omega_y + \int_{D_{y;\theta}^i} \frac{\partial \omega_y}{\partial \theta} \\
&= \int_{\partial D_{y;\theta}^{i,a}} \mathbf{v} \righthalfcup \omega_y + \int_{D_{y;\theta}^i} \frac{\partial \omega_y}{\partial \theta}
\end{split}
\end{equation}
where the last equality follows from \eqref{eq:zero_velocity} and the definition of $\mathbf{v}$.

Let $\tilde{\omega}_y = f_{\bX|Y}(\bx|y;\theta) dx_1 \wedge \cdots \wedge dx_n$. By the definition of $\cD_{y;\theta}$ we have that for all $\theta$
\begin{equation*}
\alpha = \sum_{y\in \cY}f_Y(y;\theta) \int_{\cD_{y;\theta}} \tilde{\omega}_y,
\end{equation*}
therefore, by taking a derivative, and using \eqref{eq:zero_velocity} we have
\begin{equation}\label{eq:proof2}
\begin{split}
0 = \Dth \left(\sum_{y\in \cY}f_Y(y;\theta) \int_{\cD_{y;\theta}} \tilde{\omega}_y \right) =&
\sum_{y\in \cY}\Dt{f_Y(y;\theta)} \int_{\cD_{y;\theta}} \tilde{\omega}_y \\
&+\sum_{y\in \cY}f_Y(y;\theta) \sum_{i=1}^{L_{y;\theta}} \left(\int_{\partial D_{y;\theta}^{i,a}} \mathbf{v} \righthalfcup \tilde{\omega}_y + \int_{D_{y;\theta}^i} \frac{\partial \tilde{\omega}_y}{\partial \theta} \right)
\end{split}
\end{equation}

From \eqref{eq:proof_def_a}, and linearity of the interior product we have
\begin{equation*}
\int_{\partial D_{y;\theta}^{i,a}} \mathbf{v} \righthalfcup \omega_y = \var (R;\theta) \int_{\partial D_{y;\theta}^{i,a}} \mathbf{v} \righthalfcup \tilde{\omega}_y,
\end{equation*}
therefore, plugging in \eqref{eq:proof2} we have
\begin{equation}\label{eq:proof3}
\begin{split}
\sum_{y\in \cY}f_Y(y;\theta) \sum_{i=1}^{L_{y;\theta}}\int_{\partial D_{y;\theta}^{i,a}} \mathbf{v} \righthalfcup \omega_y =& -\var (R;\theta) \sum_{y\in \cY}f_Y(y;\theta) \sum_{i=1}^{L_{y;\theta}}\int_{D_{y;\theta}^i} \frac{\partial \tilde{\omega}_y}{\partial \theta} \\
&  -\var (R;\theta) \sum_{y\in \cY}\Dt{f_Y(y;\theta)} \int_{\cD_{y;\theta}} \tilde{\omega}_y
\end{split}
\end{equation}
Now, note that from \eqref{eq:proof_cvar_derivtive} and \eqref{eq:proof_omega_leibnitz} we have
\begin{equation*}
\begin{split}
\Dth \cvar (R;\theta) =& \frac{1}{\alpha} \sum_{y\in \cY}\Dt{f_Y(y;\theta)} \sum_{i=1}^{L_{y;\theta}}\int_{\bx \in D_{y;\theta}^i} \omega_y \\
&+ \frac{1}{\alpha} \sum_{y\in \cY}f_Y(y;\theta) \sum_{i=1}^{L_{y;\theta}} \int_{D_{y;\theta}^i} \frac{\partial \omega_y}{\partial \theta}\\
&+ \frac{1}{\alpha} \sum_{y\in \cY}f_Y(y;\theta) \sum_{i=1}^{L_{y;\theta}} \int_{\partial D_{y;\theta}^{i,a}} \mathbf{v} \righthalfcup \omega_y,
\end{split}
\end{equation*}
and by plugging in \eqref{eq:proof3} we obtain
\begin{equation*}
\begin{split}
\Dth \cvar (R;\theta) =& \frac{1}{\alpha} \sum_{y\in \cY}\Dt{f_Y(y;\theta)} \sum_{i=1}^{L_{y;\theta}}\int_{D_{y;\theta}^i} \omega_y  -\var (R;\theta) \tilde{\omega}_y\\
&+ \frac{1}{\alpha} \sum_{y\in \cY}f_Y(y;\theta) \sum_{i=1}^{L_{y;\theta}} \int_{D_{y;\theta}^i} \frac{\partial \omega_y}{\partial \theta} -\var (R;\theta) \frac{\partial \tilde{\omega}_y}{\partial \theta}.
\end{split}
\end{equation*}

Finally, using the standard likelihood ratio trick -- multiplying and dividing
by $f_{Y}\left(y;\theta\right)$ inside the first sum, and multiplying and dividing
by $f_{\bX|Y}\left(\bx|y;\theta\right)$ inside the second integral we obtain the required expectation.
\end{proof}

\section{Proof of Theorem \ref{thm:consistent}}\label{sec:proof:consistent}
\begin{proof}
Let $\nu = \var(R;\theta)$. To simplify notation, we also introduce the functions
$h_1(\bx,y) \doteq \left(\Dt{ \log f_{Y}\left(y;\theta\right)} + \Dt{ \log f_{\bX|Y}\left(\bx|y;\theta\right)}\right)r(\bx,y)$, and
$h_2(\bx,y) \doteq \left(\Dt{ \log f_{Y}\left(y;\theta\right)} + \Dt{ \log f_{\bX|Y}\left(\bx|y;\theta\right)}\right)$. Thus we have

\begin{equation}\label{eq:proof_consistent_1}
\begin{split}
\dcvar =& \frac{1}{\alpha N}\sum_{i=1}^{N} \left(h_1(\bx_i,y_i) - h_2(\bx_i,y_i)\evar\right)\mathbf{1}_{r(\bx_{i},y_i)\leq \evar} \\
=& \frac{1}{\alpha N}\sum_{i=1}^{N} \left(h_1(\bx_i,y_i) - h_2(\bx_i,y_i)\nu\right)\mathbf{1}_{r(\bx_{i},y_i)\leq \nu} \\
&+ \frac{1}{\alpha N}\sum_{i=1}^{N} \left(h_1(\bx_i,y_i) - h_2(\bx_i,y_i)\nu\right)\left(\mathbf{1}_{r(\bx_{i},y_i)\leq \evar} - \mathbf{1}_{r(\bx_{i},y_i)\leq \nu}\right)\\
&+ \left( \nu - \evar \right) \frac{1}{\alpha N} \sum_{i=1}^{N} h_2(\bx_i,y_i)\left(\mathbf{1}_{r(\bx_{i},y_i)\leq \evar} \right)\\
\end{split}
\end{equation}
We furthermore let $D(\bx,y) \doteq h_1(\bx,y) - h_2(\bx,y)\nu$.
Note that by Assumption \ref{assumption:bounded_cont_rZ}, $D$ is bounded.

By Proposition \ref{prop:grad_general}, and the strong law of large numbers, we have that w.p. 1
\begin{equation}\label{eq:proof_consistent_2}
\frac{1}{\alpha N} \sum_{i=1}^{N} \left( h_1(\bx_i,y_i) - h_2(\bx_i,y_i)\nu\right)\mathbf{1}_{r(\bx_{i},y_i)\leq \nu} \to \Dth\cvar (R;\theta).
\end{equation}
We now show that the two additional terms in \eqref{eq:proof_consistent_1} vanish as $N\to \infty$. By H\"{o}lder's inequality
\begin{equation}\label{eq:proof_consistent_3}
\left| \frac{1}{N}\sum_{i=1}^{N} D(\bx_i,y_i) \left(\mathbf{1}_{r(\bx_{i},y_i)\leq \evar} - \mathbf{1}_{r(\bx_{i},y_i)\leq \nu}\right) \right|
 \leq \left( \frac{1}{N}\sum_{i=1}^{N} \left|D(\bx_i,y_i)\right|^2\right)^{0.5}
\cdot \left(\frac{1}{N}\sum_{i=1}^{N} \left|\mathbf{1}_{r(\bx_{i},y_i)\leq \evar} - \mathbf{1}_{r(\bx_{i},y_i)\leq \nu}\right|^2\right)^{0.5},
\end{equation}
and $\left( \frac{1}{N}\sum_{i=1}^{N} \left|D(\bx_i,y_i)\right|^2\right)^{0.5}$ is bounded. Also, note that
\begin{equation*}
\begin{split}
\frac{1}{N}\sum_{i=1}^{N} \left|\mathbf{1}_{r(\bx_{i},y_i)\leq \evar} - \mathbf{1}_{r(\bx_{i},y_i)\leq \nu}\right|^2 &= \frac{1}{N}\sum_{i=1}^{N} \left|\mathbf{1}_{r(\bx_{i},y_i)\leq \evar} - \mathbf{1}_{r(\bx_{i},y_i)\leq \nu}\right| \\
&= \left( \mathbf{1}_{\evar \leq \nu} - \mathbf{1}_{\nu \leq \evar}\right) \frac{1}{N}\sum_{i=1}^{N} \left(\mathbf{1}_{r(\bx_{i},y_i)\leq \evar} - \mathbf{1}_{r(\bx_{i},y_i)\leq \nu}\right)
\end{split}
\end{equation*}
By Proposition 4.1 of \citep{hong_simulating_2009}, we have that w.p. 1
\begin{equation*}
\left( \mathbf{1}_{\evar \leq \nu} - \mathbf{1}_{\nu \leq \evar}\right) \frac{1}{N}\sum_{i=1}^{N} \left(\mathbf{1}_{r(\bx_{i},y_i)\leq \evar} - \mathbf{1}_{r(\bx_{i},y_i)\leq \nu}\right) \to 0.
\end{equation*}
By the continuous mapping theorem, we thus have that w.p. 1
\begin{equation*}
\left(\frac{1}{N}\sum_{i=1}^{N} \left|\mathbf{1}_{r(\bx_{i},y_i)\leq \evar} - \mathbf{1}_{r(\bx_{i},y_i)\leq \nu}\right|^2 \right)^{0.5} \to 0,
\end{equation*}
therefore, using Eq. \eqref{eq:proof_consistent_3} we have that w.p. 1
\begin{equation}\label{eq:proof_consistent_4}
\frac{1}{N}\sum_{i=1}^{N} D(\bx_i,y_i) \left(\mathbf{1}_{r(\bx_{i},y_i)\leq \evar} - \mathbf{1}_{r(\bx_{i},y_i)\leq \nu}\right) \to 0.
\end{equation}
We now turn to the last sum in \eqref{eq:proof_consistent_1}. by Assumption \ref{assumption:bounded_cont_rZ}, $h_2$ is bounded, and therefore $\frac{1}{\alpha N} \sum_{i=1}^{N} h_2(\bx_i,y_i)\left(\mathbf{1}_{r(\bx_{i},y_i)\leq \evar} \right)$ is bounded. It is well-known \cite{david1981order} that the sample-quantile is a consistent estimator, thus $\nu - \evar \to 0$, and therefore
\begin{equation}\label{eq:proof_consistent_5}
\left( \nu - \evar \right) \frac{1}{\alpha N} \sum_{i=1}^{N} h_2(\bx_i,y_i)\left(\mathbf{1}_{r(\bx_{i},y_i)\leq \evar} \right) \to 0.
\end{equation}

Plugging \eqref{eq:proof_consistent_2}, \eqref{eq:proof_consistent_4}, and \eqref{eq:proof_consistent_5} in \eqref{eq:proof_consistent_1} gives the stated result.
\end{proof}

\section{Proof of Theorem \ref{thm:bias_bound}}\label{sec:proof:bias_bound}
We follow the notation of Section \ref{sec:proof:consistent}.

In our analysis we use a result of \citep{hong_simulating_2009}, which we now state.
Let $\bx_{1},y_1,\dots,\bx_{N},y_N$ be $N$ samples drawn i.i.d. from $f_{\bX,Y}(\bx,y;\theta)$.
\begin{theorem}{(Theorem 4.2 of \citep{hong_simulating_2009})}\label{thm:hong_liu_bound}
Let Assumption \ref{assumption:nice_f_g}, and the assumptions required for Proposition \ref{prop:grad_general} hold.
Let
\begin{equation*}
\bar{\Delta}_N = \frac{1}{\alpha N}\sum_{i=1}^{N} D(\bx_{i},y_i) \cdot \mathbf{1}_{r(\bx_{i},y_i)\leq \evar}.
\end{equation*}
Then $\E \left[ \bar{\Delta}_N \right] - \Dth \cvar (R;\theta)$ is $o(N^{-1/2})$.
\end{theorem}
In the original theorem of \citep{hong_simulating_2009}, $D$ is defined differently, corresponding to the perturbation analysis type gradient estimator. However, the proof of the theorem follows through also with our definition of $D$, and using Proposition \ref{prop:grad_general}.

We are now ready to prove Theorem \ref{thm:bias_bound}.
\begin{proof}
From Eq. \eqref{eq:proof_consistent_1} we have
\begin{equation}\label{eq:proof_bias_1}
\begin{split}
\E \left[ \dcvar \right] - \Dth\cvar (R;\theta) = &\E \left[ \frac{1}{\alpha N}\sum_{i=1}^{N} D(\bx_i,y_i) \cdot \mathbf{1}_{r(\bx_{i},y_i)\leq \evar}\right]\\
&+ \E \left[ \left( \nu - \evar \right) \frac{1}{\alpha N} \sum_{i=1}^{N} h_2(\bx_i,y_i)\left(\mathbf{1}_{r(\bx_{i},y_i)\leq \evar} \right) \right].
\end{split}
\end{equation}
The first term in the r.h.s. of Eq. \eqref{eq:proof_bias_1} is $o(N^{-1/2})$ by Theorem \ref{thm:hong_liu_bound}. We now bound the second term.

Let $\bar{h}_2$ denote a bound on $h_2$, which, by Assumption \ref{assumption:bounded_cont_rZ}, is finite. Note that we have $\left| \frac{1}{N} \sum_{i=1}^{N} h_2(\bx_i,y_i)\left(\mathbf{1}_{r(\bx_{i},y_i)\leq \evar} \right) \right| \leq \bar{h}_2$ with probability 1.
Therefore,
\begin{equation*}\label{eq:proof_bias_2}
\E \left[ \left( \nu - \evar \right) \frac{1}{\alpha N} \sum_{i=1}^{N} h_2(\bx_i,y_i)\left(\mathbf{1}_{r(\bx_{i},y_i)\leq \evar} \right) \right] \leq \E \left[ \left| \nu - \evar \right| \left| \frac{1}{\alpha N} \sum_{i=1}^{N} h_2(\bx_i,y_i)\left(\mathbf{1}_{r(\bx_{i},y_i)\leq \evar} \right) \right| \right] \leq \frac{\bar{h}_2}{\alpha} \E \left[ \left| \nu - \evar \right| \right].
\end{equation*}

We will show that $\E \left[ \left| \nu - \evar \right| \right]$ is $O(N^{-1/2})$. 
It is well-known \citep{david1981order} that the empirical $\alpha-$quantile may be written as follows:
\begin{equation}\label{eq:proof_bias_3}
\evar = \nu - \frac{\hat{F}_R(\nu) - \alpha}{f_R(\nu)} + \tilde{R},
\end{equation}
where $\hat{F}_R(\cdot)$ is the empirical C.D.F. of $R$, and $\tilde{R}$ is $O(N^{-1/2})$ in probability. Thus, we have
\begin{equation}\label{eq:proof_bias_4}
\E \left[ \left| \nu - \evar \right| \right] \leq f_R(\nu)^{-1}\left(\E \left[ \left|\hat{F}_R(\nu) - \alpha\right| \right] + \E \left[ \left| \tilde{R} \right| \right]\right).
\end{equation}
Note that since $R$ is bounded, $\evar$ is also bounded, and it is clear from Eq. \eqref{eq:proof_bias_3} that $\tilde{R}$ is bounded, and therefore uniformly integrable. Since $\tilde{R}$ is also $O(N^{-1/2})$ in probability, we conclude that $\E \left[ \left| \tilde{R} \right| \right]$ is $O(N^{-1/2})$.
Let $y_i \doteq \mathbf{1}_{r(\bx_{i},y_i)\leq \nu}$. Then by definition, the empirical C.D.F. satisfies
\begin{equation*}
\hat{F}_R(\nu) = \frac{1}{N} \sum_{i=1}^{N} y_i,
\end{equation*}
and the $y_i$'s are i.i.d., and satisfy $\E [ y_i ] = \alpha$, and $\variance [y_i] = \alpha (1- \alpha)$.
Observe that
\begin{equation*}
0 \leq \variance \left[ \left|\hat{F}_R(\nu) - \alpha\right|\right] = \E \left[ \left|\hat{F}_R(\nu) - \alpha\right|^2\right] - \left( \E \left[ \left|\hat{F}_R(\nu) - \alpha\right|\right] \right)^2,
\end{equation*}
therefore
\begin{equation*}
\E \left[ \left|\hat{F}_R(\nu) - \alpha\right|\right] \leq \sqrt{\E \left[ \left|\hat{F}_R(\nu) - \alpha\right|^2\right]},
\end{equation*}
but
\begin{equation*}
\E \left[ \left|\hat{F}_R(\nu) - \alpha\right|^2\right] = \variance \left[ \hat{F}_R(\nu) \right] = \frac{\alpha (1- \alpha)}{N},
\end{equation*}
therefore $\E \left[ \left|\hat{F}_R(\nu) - \alpha\right|\right]$ is $O(N^{-1/2})$. From Eq. \eqref{eq:proof_bias_4} we thus have that $\E \left[ \left| \nu - \evar \right| \right]$ is $O(N^{-1/2})$, which completes the proof.
\end{proof}

\section{Example: the Importance of the VaR Baseline in \texttt{GCVaR}}\label{sec:bias_example}
Here we show that the subtraction of the VaR baseline from the reward in \texttt{GCVaR} (Eq. \eqref{eq:simple_alg}) is crucial, and without it the error in the gradient estimate may be arbitrarily large.

Consider the following example, in the setting of proposition \ref{prop:Grad}: $Z \sim Normal(\theta,1)$, and $\alpha=0.5$. The true CVaR gradient is constant:
\begin{equation*}
\Dth \! \cvar (\!Z;\!\theta) = \E^{\theta}\left[ \!\left.\Dt{ \log \! f_Z(Z;\theta)}(Z - \var (Z;\theta))\right|\!Z \!\leq\! \var \!(Z;\!\theta)\! \right] = 1,
\end{equation*}
while the term due to the baseline is
\begin{equation*}
\E^{\theta}\left[ \!\left.\Dt{ \log \! f_Z(Z;\theta)}(- \var (Z;\theta))\right|\!Z \!\leq\! \var \!(Z;\!\theta)\! \right] = -\sqrt{\frac{2}{\pi}}\theta,
\end{equation*}
which is unbounded in $\theta$.

Thus, we have that $\E^{\theta}\left[ \!\left.\Dt{ \log \! f_Z(Z;\theta)}(Z)\right|\!Z \!\leq\! \var \!(Z;\!\theta)\! \right] = 1 + \sqrt{\frac{2}{\pi}}\theta$, meaning that a naive estimator without the baseline may have an arbitrarily large error, and, for $\theta<-\sqrt{\frac{\pi}{2}}$, would even point in the opposite direction!

\section{Importance Sampling}\label{sec:IS}
For very low quantiles, i.e., $\alpha$ close to $0$, the estimator \texttt{GCVaR} of Eq. \eqref{eq:simple_alg} would have a high variance, since the averaging is effectively only over $\alpha N$ samples. In order to mitigate this problem, we now propose an importance sampling procedure for estimating $\Dth\cvar (R;\theta)$.

Importance sampling (IS; \citep{rubinstein2011simulation}) is a general procedure for reducing the variance of Monte--Carlo (MC) estimates. We first describe it in a general context, and then give the specific implementation for the CVaR sensitivity estimator.

\subsection{Background}\label{sec:IS_background}
Consider the following general problem. We wish to estimate the expectation $l = \E \left[ H(X)\right]$ where $X$ is a random variable with P.D.F. $f(x)$, and $H(x)$ is some function. The MC solution is given by $\hat{l} = \frac{1}{N}\sum_{i=1}^{N}H(x_i)$, where $x_i \sim f$ are drawn i.i.d.

The IS method aims to reduce the variance of the MC estimator by using a different sampling distribution for the samples $x_i$. Assume we are given a sampling distribution $g(x)$, and that $g$ dominates $f$ in the sense that $g(x) = 0 \Rightarrow f(x) = 0$. We let $\E_f$ and $\E_g$ denote expectations w.r.t. $f$ and $g$, respectively. Observe that
$
l = \E_f \left[ H(X)\right] = \E_g \left[ H(X)\frac{f(X)}{g(X)}\right],
$
and we thus define the IS estimator $\hat{l}_{\is}$ as
\begin{equation}\label{eq:IS_general}
\hat{l}_{\is} = \frac{1}{N}\sum_{i=1}^{N}H(x_i)\frac{f(x_i)}{g(x_i)},
\end{equation}
where the $x_i$'s are drawn i.i.d., and now $x_i \sim g$. Obviously, selecting an appropriate $g$ such that $\hat{l}_{\is}$ indeed has a lower variance than $\hat{l}$ is the heart of the problem. One approach is by the \emph{variance minimization} method \cite{rubinstein2011simulation}. Here, we are given a family of distributions $g(x;\omega)$ parameterized by $\omega$, and we aim to find an $\omega$ that minimizes the variance $V(\omega) = \variance_{x_i\sim g(\cdot;\omega)} \left( \hat{l}_{\is} \right)$. A straightforward calculation shows that
$V(\omega) = \E_f \left[ H(X)^2 \frac{f(X)}{g(X;\omega)}\right] - l^2$, and since $l$ does not depend on $\omega$, we are left with the optimization problem
$
\min_\omega \E_f \left[ H(X)^2 \frac{f(X)}{g(X;\omega)}\right],
$
which is typically solved approximately, by solving the sampled average approximation (SAA)
\begin{equation}\label{eq:IS_SAA}
\min_\omega \frac{1}{N_{\textrm{SAA}}}\sum_{i=1}^{N_{\textrm{SAA}}} \left[ H(x_i)^2 \frac{f(x_i)}{g(x_i;\omega)}\right],
\end{equation}
where $x_i \sim f$ are i.i.d. Numerically, the SAA may be solved using (deterministic) gradient descent, by noting that $\Dom \left( \frac{f(x_i)}{g(x_i;\omega)} \right) = - \frac{f(x_i)}{g(x_i;\omega)} \Dom \log g(x_i;\omega)$.

Thus, in order to find an IS distribution $g$ from a family of distributions $g(x;\omega)$, we draw $N_{\textrm{SAA}}$ samples from the original distribution $f$, and solve the SAA \eqref{eq:IS_SAA} to obtain the optimal $\omega$. We now describe how this procedure is applied for estimating the CVaR sensitivity $\Dth\cvar (R;\theta)$.

\subsection{IS Estimate for CVaR Sensitivity}\label{sec:IS_CVaR}
We recall the setting of Proposition \ref{prop:grad_general}, and assume that in addition to $f_{\bX,Y}(\bx,y;\theta)$ we have access to a family of distributions $g_{\bX,Y}\left(\bx,y;\theta,\omega\right)$ parameterized by $\omega$. We follow the procedure outlined above and, using Proposition \ref{prop:grad_general}, set
\begin{equation*}
H_j(\bX,Y) = \frac{1}{\alpha} \left(\frac{\partial {\log f_Y(Y;\theta)}}{\partial \theta_j}+\frac{\partial {\log f_{\bX|Y}(\bX|Y;\theta)}}{\partial \theta_j}\right)\left(R - \var (R;\theta)\right) \mathbf{1}_{R\leq \var (R;\theta)}.
\end{equation*}
However, since $\var (R;\theta)$ is not known in advance, we need a procedure for estimating it in order to plug it into Eq. \eqref{eq:IS_general}. The empirical quantile $\evar$ of Eq. \eqref{eq:empirical_VaR} is not suitable since it uses samples from $f_{\bX,Y}(\bx,y;\theta)$. Thus, we require an IS estimator for $\var (R;\theta)$ as well. Such was proposed by \citet{glynn1996importance}. Let $\hat{F}_{\is}(z)$ denote the IS empirical C.D.F. of $R$:
$
\hat{F}_{\is}(z) \doteq \frac{1}{N} \sum_{i=1}^{N} \frac{f_{\bX,Y}(\bx_i,y_i;\theta)}{g_{\bX,Y}\left(\bx_i,y_i;\theta,\omega\right)} \mathbf{1}_{r(\bx_{i},y_i)\leq z}.
$
Then, the IS empirical VaR is given by
\begin{equation}\label{eq:IS_empirical_VaR}
\evar_{\is} = \inf_z \hat{F}_{\is}(z) \geq \alpha.
\end{equation}
We also need to modify the variance minimization method, as we are not estimating a scalar function but a gradient in $\R^k$. We assume independence between the elements, and replace $H(x_i)^2$ in Eq. \eqref{eq:IS_SAA} with $\sum_{j=1}^{k} H_j(x_i)^2$.

Let us now state the estimation procedure explicitly. We first draw $N_{\textrm{SAA}}$ i.i.d. samples from $f_{\bX,Y}(\bx,y;\theta)$, and find a suitable $\omega$ by solving the following equivalent of \eqref{eq:IS_SAA}

\begin{equation}\label{eq:IS_SAA_explicit}
\min_\omega \frac{1}{N_{\textrm{SAA}}}\sum_{i=1}^{N_{\textrm{SAA}}} \left[ \sum_{j=1}^{k} H_j(\bx_i,y_i)^2 \frac{f_{\bX,Y}(\bx_i,y_i;\theta)}{g_{\bX,Y}\left(\bx_i,y_i;\theta,\omega\right)}\right],
\end{equation}

with $H_j(\bX,Y) = \frac{1}{\alpha} \left(\frac{\partial {\log f_Y(Y;\theta)}}{\partial \theta_j}+\frac{\partial {\log f_{\bX|Y}(\bX|Y;\theta)}}{\partial \theta_j}\right) (r(\bX,Y)-\evar) \mathbf{1}_{r(\bX,Y)\leq \evar}$, where $\evar$ is given in \eqref{eq:empirical_VaR}.

We then run the \texttt{IS\_GCVaR} algorithm, as follows. We draw $N$ i.i.d. samples $\bx_{1},y_1,\dots,\bx_{N},y_N$ from $g_{\bX,Y}\left(\bx,y;\theta,\omega\right)$. The IS estimate of the CVaR gradient $\dcvarIS$ is given by
\begin{equation}\label{eq:IS_alg}
\dcvarIS =
\frac{1}{\alpha N}\sum_{i=1}^{N} \left(\frac{\partial {\log f_Y(y_i;\theta)}}{\partial \theta_j}+\frac{\partial {\log f_{\bX|Y}(\bx_i|y_i;\theta)}}{\partial \theta_j}\right)\frac{f_{\bX,Y}(\bx_i,y_i;\theta) (r(\bx_{i},y_i)-\evar_{\is}) \mathbf{1}_{r(\bx_{i},y_i)\leq \evar_{\is}}}{g_{\bX,Y}\left(\bx_i,y_i;\theta,\omega\right)},
\end{equation}
where $\evar_{\is}$ is given in \eqref{eq:IS_empirical_VaR}.

\begin{algorithm} \label{alg:IS_GCVaR}
\caption{\texttt{IS\_GCVaR}}
1: \alggiven
\begin{itemize}
\item CVaR level $\alpha$
\item A reward function $r(\bx,y):\R^n\bigotimes \cY \to \R \to \R$
\item A density function $f_{\bX,Y}(\bx,y;\theta)$
\item A density function $g_{\bX,Y}\left(\bx,y;\theta\right)$
\item A sequence $\bx_{1},y_1,\dots,\bx_{N},y_N \sim g_{\bX,Y}$, i.i.d.
\end{itemize}
2: Set $\bx_1^s,y_1^s\dots,\bx_N^s,y_N^s = \textrm{Sort}\left(\bx_{1},y_1,\dots,\bx_{N},y_N\right)$ by $r(\bx,y)$

3: For $i = 1,\dots,N$ do
\begin{equation*}
L(i) = \sum_{j=1}^{i} f_{\bX,Y}\left(\bx^s_j,y^s_j;\theta\right) / g_{\bX,Y}\left(\bx^s_j,y^s_j;\theta\right)
\end{equation*}

4: Set $l = \argmin_i L(i) \geq \alpha $

5: Set $\evar_{\is} = r(\bx_{l}^s,y_l^s)$

6: For $j = 1,\dots,k$ do
\begin{equation*}
\begin{split}
&\dcvarIS =
\frac{1}{\alpha N}\sum_{i=1}^{N} \left(\frac{\partial {\log f_Y(y_i;\theta)}}{\partial \theta_j}+\frac{\partial {\log f_{\bX|Y}(\bx_i|y_i;\theta)}}{\partial \theta_j}\right)\frac{f_{\bX,Y}(\bx_i,y_i;\theta) (r(\bx_{i},y_i)-\evar_{\is}) \mathbf{1}_{r(\bx_{i},y_i)\leq \evar_{\is}}}{g_{\bX,Y}\left(\bx_i,y_i;\theta\right)},
\end{split}
\end{equation*}
7: \algreturn $\Delta^{\is}_{1;N},\dots,\Delta^{\is}_{k;N}$
\end{algorithm}

Note that in our SAA program for finding $\omega$, we estimate $\var$ using crude Monte Carlo. In principle, IS may also be used for that estimate as well, with an additional optimization process for finding a suitable sampling distribution. However, a typical application of the CVaR gradient is in optimization of $\theta$ by stochastic gradient descent. There, one only needs to update $\omega$ intermittently, therefore a large sample size $N_{\textrm{SAA}}$ is affordable and IS is not needed.

So far, we have not discussed how the parameterized distribution family $g_{\bX,Y}\left(\bx,y;\theta,\omega\right)$ is obtained. While there are some standard approaches such as exponential tilting \cite{rubinstein2011simulation}, this task typically requires some domain knowledge.
For the RL domain, we present a heuristic method for selecting $g_{\bX,Y}\left(\bx,y;\theta,\omega\right)$.

\subsection{CVaR Policy Gradient with Importance Sampling}\label{sec:CVaR_PG_IS}

As explained earlier, when dealing with small values of $\alpha$, an IS scheme may help reduce the variance of the CVaR gradient estimator. In this section, we apply the IS estimator to the RL domain. As is typical in IS, the main difficulty is finding a suitable sampling distribution, and actually sampling from it. In RL, a natural method for modifying the trajectory distribution is by modifying the MDP transition probabilities. We note, however, that by such our method actually requires access to a simulator of this modified MDP. In many applications a simulator of the original system is available anyway, thus modifying it should not be a problem.

Consider the RL setting of Section \ref{sec:RL}, and denote the original MDP by $M$. The P.D.F. of a trajectory $\{X,Y\}$ from the MDP $M$, where, as defined in the main text $Y \doteq s_0,a_0,s_1,a_1,\dots,s_\tau, X \doteq \rho_0,\rho_1,\dots,\rho_{\tau-1}$ is given by
\begin{equation*}
f_{\bX,Y}(\bx,y;\theta) =
\zeta_0(s_0) \prod_{t=0}^{\tau-1} f_{a|s}(a_t|s_t;\theta) f_{\rho|s,a} (\rho|s_t,a_t) f_{s'|s,a}(s_{t+1}|s_t,a_t).
\end{equation*}
Consider now an MDP $\hat{M}$ that is similar to the original MDP $M$ but with transition probabilities $\hat{f}_{s'|s,a}(s'|s,a;\omega)$, where $\omega$ is some controllable parameter. We will later specify $\hat{f}_{s'|s,a}(s'|s,a;\omega)$ explicitly, but for now, observe that the P.D.F. of a trajectory $\{X,Y\}$ from the MDP $\hat{M}$ is given by
\begin{equation*}
g_{\bX,Y}(\bx,y;\theta,\omega) =
\zeta_0(s_0) \prod_{t=0}^{\tau-1} f_{a|s}(a_t|s_t;\theta) f_{\rho|s,a} (\rho|s_t,a_t) \hat{f}_{s'|s,a}(s_{t+1}|s_t,a_t;\omega).
\end{equation*}
and therefore
\begin{equation}\label{eq:MDP_f_g}
\frac{f_{\bX,Y}(\bx,y;\theta)}{g_{\bX,Y}(\bx,y;\theta,\omega)} = \prod_{t=0}^{\tau-1} \frac{f_{s'|s,a}(s_{t+1}|s_t,a_t)}{\hat{f}_{s'|s,a}(s_{t+1}|s_t,a_t;\omega)}.
\end{equation}
Using Eq. \eqref{eq:MDP_dlogf}, Eq. \eqref{eq:MDP_f_g}, and the fact that $\partial \log f_{\bX|Y}\left(x_{i}|y_i;\theta\right) \!/\! \partial \theta\!=\!0$ in our formulation, the \texttt{IS\_GCVaR} algorithm may be used to obtain the IS estimated gradient $\dcvarIS$, which may then be used instead of $\dcvar$ in the parameter update equation \eqref{eq:theta_update_SGD}.

We now turn to the problem of choosing the transition probabilities $\hat{f}_{s'|s,a}(s'|s,a;\omega)$ in the MDP $\hat{M}$, and propose a heuristic approach that is suitable for the RL domain. We first observe that by definition, the CVaR takes into account only the `worst' trajectories for a given policy, therefore a suitable IS distribution should give more weight to such bad outcomes in some sense. The difficulty is how to modify the transition probabilities, which are defined per state, such that the whole trajectory will be `bad'.
We note that this difficulty is in a sense opposite to the action selection problem: how to choose an action at each state such that the long-term reward is high. Action selection is a fundamental task in RL, and has a very elegant solution, which inspires our IS approach.

A standard approach to action selection is through the \emph{value-function} $V(s)$ \cite{sutton_reinforcement_1998}, which assigns to each state $s$ its expected long term outcome $\E\left[B|s_0 = s\right]$ under the current policy. Once the value function is known, the `greedy selection' rule selects the action that maximizes the expected value of the next state.
The intuition behind this rule is that since $V(s)$ captures the long-term return from $s$, states with higher values lead to better trajectories, and should be preferred.
\newline
By a similar reasoning, we expect that encouraging transitions to low-valued states will produce worse trajectories. We thus propose the following heuristic for the transition probabilities $\hat{f}_{s'|s,a}(s'|s,a;\omega)$. Assume that we have access to an approximate value function $\tilde{V}(s)$ for each state. We propose the following IS transitions for $\hat{M}$
\begin{equation}\label{eq:IS_MDP_heuristic}
\hat{f}_{s'|s,a}(s'|s,a;\omega) = \frac{f_{s'|s,a}(s'|s,a) \exp \left( -\omega \tilde{V}(s';\theta)\right)}{\sum_y f_{s'|s,a}(y|s,a) \exp \left( -\omega \tilde{V}(y;\theta)\right)}.
\end{equation}
Note that increasing $\omega$ encourages transitions to low value states, thus increasing the probability of `bad' trajectories.

Obtaining an approximate value function for a given policy has been studied extensively in RL literature, and many efficient solutions for this task are known, such as LSTD \citep{boyan2002technical} and TD($\lambda$) \citep{sutton_reinforcement_1998}. Here, we don't restrict ourselves to a specific method.

\subsection{Empirical Results with Importance Sampling}\label{sec:CVaR_PG_IS}

We report the full details about the experimental results with importance sampling mentioned in the main text.

Fig. \ref{IS} demonstrates the importance of IS in optimizing the CVaR when $\alpha$ is small. We chose $\alpha=0.01$, and $N = 200$, and compared the naive \texttt{GCVaR} against \texttt{IS\_GCVaR}. As our value function approximation, we exploited the fact that the soft-max policy uses $\phi(s,a)^{\T} \theta$ as a sort of state-action value function, and therefore set
$\tilde{V}(s) = \max_a \phi(s,a)^{\T} \theta$. We chose $\omega$ using SAA, with trajectories from the initial policy $\theta_0$.
We observe that \texttt{IS\_GCVaR} converges significantly faster than \texttt{GCVaR}, due to the lower variance in gradient estimation.

\begin{figure}
\vskip 0.2in
\begin{center}
\centerline{
\includegraphics[scale=0.33, clip=true]{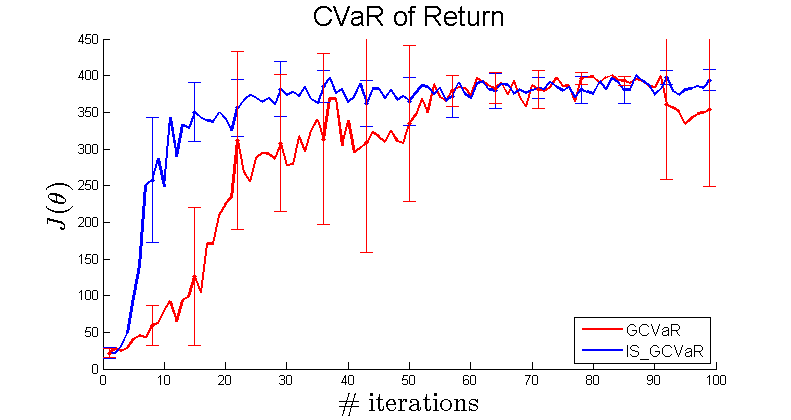}
}
\caption{\textbf{\texttt{IS\_GCVaR} vs. \texttt{GCVaR}} CVaR ($\alpha=0.01$) of the return for \texttt{IS\_GCVaR} and \texttt{GCVaR} vs. iteration.}
\label{IS}
\end{center}
\vskip -0.2in
\end{figure}

\end{document}